%% file: hfr2020.tex
\documentclass{svproc}

\input{packages_commands/mypackages}
\input{packages_commands/mycommands}
\usepackage{url} 
\usepackage{hyperref}

\begin{document}
\mainmatter
\title{A Safety and Passivity Filter\\
for Robot Teleoperation Systems}
\titlerunning{A Safety and Passivity Filter for Robot Teleoperation Systems}
\author{Gennaro Notomista\inst{1} \and Xiaoyi Cai\inst{2}}
\authorrunning{Gennaro Notomista and Xiaoyi Cai}
\tocauthor{Gennaro Notomista and Xiaoyi Cai}
\institute{Georgia Institute of Technology, Atlanta, GA 30308, USA,\\
\email{g.notomista@gatech.edu}
\and Massachusetts Institute of Technology, Cambridge, MA 02139, USA,\\
\email{xyc@mit.edu}
}

\maketitle

\begin{abstract}
In this paper, we present a way of enforcing safety and passivity properties of robot teleoperation systems, where a human operator interacts with a dynamical system modeling the robot. The approach does so in a holistic fashion, by combining safety and passivity constraints in a single optimization-based controller which effectively \textit{filters} the desired control input before supplying it to the system. The result is a safety and passivity filter implemented as a convex quadratic program which can be solved efficiently and employed in an online fashion in many robotic teleoperation applications. Simulation results show the benefits of the approach developed in this paper applied to the human teleoperation of a second-order dynamical system.

\keywords{Robot teleoperation, safety of dynamical systems, passivity of dynamical systems, control barrier functions, integral control barrier functions}
\end{abstract}

\section{Introduction}

In robot teleoperation, the robot and a human operator can be seen as interconnected systems that exchange inputs and outputs. The dynamics of these systems, as well as that of the communication channel between them, can lead to unpredictable behaviors of the compound system. Therefore, it is often convenient analyzing robot teleoperation systems from an energetic point of view, which consists of keeping track of the energy the interconnected systems exchange between each other. \textit{Passivity}-based approaches to the control of interconnected systems \cite{hatanaka2015passivity} have demonstrated to be suitable in many application domains, ranging from telemanipulation \cite{stramigioli2002geometric,niemeyer2004telemanipulation,sieber2018human} to teleoperation of multi-robot systems \cite{chopra2006passivity}.

Passivity is an amenable property as it ensures the energy generated by the system does not exceed the one injected through the input to the system. Energy injected from external sources or other interconnected systems can make a system become non-passive, as discussed in~\cite{anderson1989bilateral}. Passivity theory allows us to analyze dynamical systems from an energetic point of view, so it is a very suitable design tool for dealing with interconnected systems and, therefore with robot teleoperation. Additionally, energy considerations can be useful to account for delays in the communication between a human teleoperator and a robotic system that is remotely controlled \cite{zampieri2008trends}, which can cause the performance of the algorithms to degrade, in terms of both convergence rate and stability \cite{olfati2004consensus}.

Passivity-based approaches, as well as other energy-based methods, for the control of robotic systems are considered in  \cite{anderson1989bilateral,niemeyer1991stable,wohlers2017lumped,yamauchi2017passivity,notomista2019passivity}. In \cite{duindam2004port}, the authors introduce the concept of \textit{energy tanks}, which is then extended in \cite{secchi2006position,secchi2007control,secchi2012bilateral,giordano2013passivity}. These works consider additional dissipative forces on a teleoperation system that has to kept passive in order to prevent energy tanks from depleting---a condition that would introduce a singularity in their proposed approaches---so as to keep a positive passivity margin, intended as the energy dissipated by the system over time.

Besides the passivity property, in many robotic applications, it is also desirable to ensure the \textit{safety} of the system, intended as the \textit{forward invariance} of a subset of the system state space. This is particularly crucial when robotic systems interact or collaborate with humans in order to perform a task. The safety of human operators can be guaranteed by constraining the robot to operate in safe regions of the workspace. To this end, control barrier functions (CBFs) \cite{ames2019control} are a control-theoretic tool which can be employed in order to ensure safety in dynamical systems.

\begin{figure}
\centering
\includegraphics[width=0.75\textwidth]{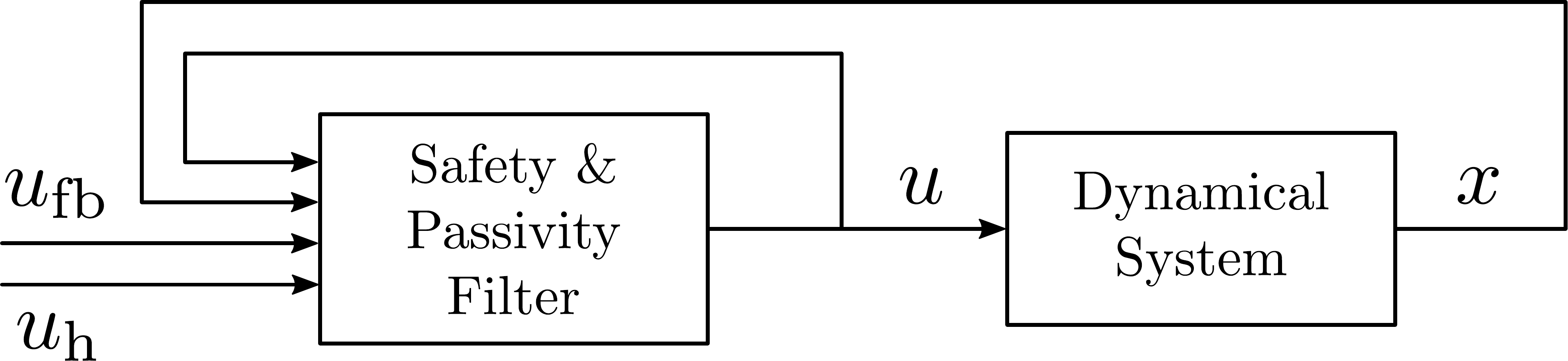}
\caption{Passivity and safety filter: the human input $u\h$ the state feedback controller $u\fb$ are modified by a filter before being supplied to the dynamical system (e.g. a teleoperated robotic system) in order to ensure its safety and passivity.}
\label{fig:passafetyfiltermin}
\end{figure}

In this paper, we propose a way of dealing with safety and passivity  objectives in a holistic fashion. We do so by introducing a \textit{safety and passivity filter} (see Fig.~\ref{fig:passafetyfiltermin}), whose goal is that of modifying the input to the robotic system in order to render it safe and passive. \textit{The filter effectively modifies the system itself in order to ensure that it remains safe and passive.} The proposed approach is able to seamlessly account for user-defined feedback control laws---which can be leveraged to endow the system with stability properties or to accomplish other objectives---combined with the inputs of a human interacting with the robot. The computational burden introduced by the designed filter is low, making the approach amenable for the real-time implementation on many robotic platforms.

\section{Background}

Throughout this paper, we consider a robotic system modeled by the following nonlinear control system:
\begin{equation}
\label{eq:sysdyn}
\left\{\begin{aligned}
\dot x &= f(x,u)\\
y &= g(x)
\end{aligned}\right.
\end{equation}
where $x\in\R^n$, $u\in\R^m$ and $y\in\R^m$ are the state, input and output, respectively, of the system, $f\colon\R^n\times\R^m \to\R^n$ is a Lipschitz continous vector fields and $g\colon\R^n\to\R^m$.

To account for both state feedback controllers and external human inputs, we explicitly consider the input $u$ broken down as follows:
\begin{equation}
\label{eq:input}
u = u\fb(x) + u\h,
\end{equation}
where $u\fb\colon\R^n\to\R^m$ denotes the state feedback component of the input, and $u\h$ represents the input given by a human operator.

As we are interested in provably guaranteeing passivity and safety properties of dynamical systems, in the following we briefly introduce this two concepts.

\begin{definition}[Passivity \cite{khalil2015nonlinear}]
\label{def:passivity}
The system \eqref{eq:sysdyn} is passive if there exists a continuously differentiable positive definite storage function ${V:\R^n\to\R}$ such that, for all $x$ and $u$,
\begin{equation}
\dot V = \dfrac{\de V}{\de x} f(x,u) \le u\tr y.
\end{equation}
The system is called lossless if $\dot V = u\tr y$.
\end{definition}

The definition of passivity is a statement about the system, which holds for all possible values of the input $u$ and the output $y$. Safety, on the other hand, is tied to the definition of a \textit{safe set}, i.e. a subset of the state space of the system where we want the state of the system to remain confined for all times. A technique which proved to be applicable to a variety of robotic systems \cite{ames2014control,wu2016safety,nguyen2016exponential,wang2017safety,ohnishi2019barrier,notomista2020persistification} and different scenarios consists of employing control barrier functions (CBFs). In the following, we introduce the definition of CBFs as in \cite{ames2019control} and the main result which will be used in this paper to ensure controlled forward invariance, i.e. controlled safety.

\begin{definition}[Control Barrier Functions (CBFs) \cite{ames2019control}]
\label{def:cbf}
Let $\mc C \subset \mc D \subset \R^n$ be the zero superlevel set of a continuously differentiable function $h\colon\mc D \to \R$. Then $h$ is a control barrier function (CBF) if there exists an extended class $\mc K_\infty$ function\footnote[1]{An extended class $\mc K_\infty$ function is a continuous function $\gamma : \R \to \R$ that is strictly increasing and with $\gamma(0) = 0$.} $\gamma$ such that, for the system \eqref{eq:sysdyn},
\begin{equation}
\label{eq:cbfdefinition}
\sup_{u \in \mc \R^m}  \left\{ L_f h(x,u) + \gamma(h(x))\right\} \geq 0.
\end{equation}
for all $ x \in \mc D$.
\end{definition}
The notation $L_f h(x)$ denotes the Lie derivative of $h$ along the vector field $f$. Given this definition of CBFs, the following theorem highlights how they can be used to ensure both set forward invariance (safety) and stability.

\begin{theorem}[Safety \cite{ames2019control}]
\label{thm:safety}
Let $\mathcal{C} \subset \R^n$ be a set defined as the zero superlevel set of a continuously differentiable function $h: \mc D \subset \R^n \to \R$.  If $h$ is a CBF on $\mc D$ with $0$ a regular value, then any Lipschitz continuous controller $u(x) \in \{ u \in \mc \R^m \colon L_f h(x,u) + \gamma(h(x)) \geq 0\}$ for the system \eqref{eq:sysdyn} renders the set $\mc C$ forward invariant (safe).  Additionally, the set $\mc C$ is asymptotically stable in $\mc D$.
\end{theorem}

Besides safety, in this paper, we are interested in enforcing passivity conditions onto a dynamical system, representing for example a robot interacting with a human in a teleoperation task. However, the condition of passivity, recalled in \ref{def:passivity}, involves the input $u$. Recently introduced integral CBFs (I-CBFs) \cite{ames2020integral}---which generalize control dependent CBFs \cite{huang2019guaranteed}---can be leveraged to enforce passivity conditions. In order to take advantage of I-CBFs, the system \eqref{eq:sysdyn} needs to be \textit{dynamically extended} as follows:
\begin{equation}
\label{eq:sysdynext}
\left\{\begin{aligned}
\dot x &= f(x,u)\\
\dot u &= \phi(x,u,t) + v\\
y &= g(x)
\end{aligned}\right.
\end{equation}
where $v\in\R^m$ is the new control input and $\phi\colon\R^n\times\R^m\times\R\to\R^m$ will be designed to ensure that $u = u\fb(x) + u\h$, as in \eqref{eq:input}, as desired. 

We now have the necessary constructions to introduce integral CBFs.
\begin{definition}[Integral Control Barrier Functions (I-CBFs)\cite{ames2020integral}]
For the system \eqref{eq:sysdyn}, with corresponding safe set $\mc S = \{(x,u) \in \R^n \times \R^m \colon h(x,u) \geq 0\} \subset \R^n \times \R^m$ defined as the zero superlevel set of a function $h : \R^n \times \R^m \to \R$ with $0$ a regular value.  Then, $h$ is an integral control barrier function (I-CBF) if for any $(x,u) \in \R^n \times \R^m$ and $t \geq 0$:
\begin{equation}
\label{eq:icbfdef}
\frac{\partial h}{\partial u}(x,u) = 0 \implies \frac{\partial h}{\partial x}(x,u)f(x, u)  + \gamma(h(x,u)) \geq 0.
\end{equation}
\end{definition}

The implication in \eqref{eq:icbfdef} guarantees that, by means of an I-CBF $h$, the inequality
\begin{equation}
\dot h(x,u) + \gamma(h(x,u)) = \frac{\partial h}{\partial x}(x,u)f(x, u)  + \frac{\partial h}{\partial u}(x,u)\phi(x,u,t) + \gamma(h(x,u)) \geq 0
\end{equation}
---which, by the comparison lemma \cite{khalil2015nonlinear}, guarantees the forward invariance of the set $\mc S$---can always be satisfied by a proper choice of $\phi(x,u,t)$.
With the definition of I-CBFs, we concluded the introduction of the control-theoretical tools employed in the next section to design an input filter to render a dynamical system safe and passive.

\section{Safety and Passivity Filter}

In this section, we develop the safety and passivity filter. We brake down its structure into three components:
\begin{enumerate}
\item Safety-preserving controller, described in Section~\ref{subsec:safety}
\item Passivity-preserving controller, described in Section~\ref{subsec:passivity}
\item Tracking controller, described in Section~\ref{subsec:tracking}
\end{enumerate}
These three components will be then combined in Section~\ref{subsec:filter} where Proposition~\ref{prop:filter} is stated, which ensures that the designed filter achieves the desired properties.

\subsection{Ensuring Safety}
\label{subsec:safety}

Define the safe set $\mc S_x$ to be the zero superlevel set of a continuously differentiable function $h_x$, i.e.
\begin{equation}
\label{eq:safeset}
\mc S_x = \{x\in\R^n \colon h_x(x)\ge0\}.
\end{equation}
We want the state $x$ of the system \eqref{eq:sysdyn} to be confined in $\mc S_x$ for all times $t$. This condition, corresponding to safety, can be enforced using Theorem~\ref{thm:safety}.

If $h_x$ has relative degree 1 with respect to the input $u$---i.e. the time derivative $\dot h_x$ explicitly depends on $u$---then it has relative degree 2 with respect to the input $v$, based on the dynamic extension described by \eqref{eq:sysdynext}. In order for Theorem~\ref{thm:safety} to be applicable to the system \eqref{eq:sysdynext}, we need $L_f h_x(x,u)$ to depend on the input $v$, a condition that does not hold if the relative degree of $h_x$ with respect to $v$ is greater than 1.

To circumvent this issue, following the idea in \cite{notomista2020persistification} or \cite{ames2020integral}, let 
\begin{equation}
\label{eq:recursivecbf}
h_x^\prime(x,u) := \dot h_x(x,u) + \gamma(h_x(x)).
\end{equation}
Since $h_x^\prime$ depends on $u$, $L_f h_x^\prime$ depends on $v$. Then, in order to ensure the safety of $S_x$, we may choose any control input $v$ satisfying the following inequality:
\begin{equation}
L_f h_x^\prime(x,u,v) + \gamma_x(h_x^\prime(x,u)) \geq 0.
\end{equation}
This way, by Theorem~\ref{thm:safety}, $h_x^\prime(x,u)\ge0$ for all times, which, by \eqref{eq:recursivecbf}, is equivalent to $\dot h_x(x,u) + \gamma(h_x(x))\ge0$ for all times. The repeated application of Theorem~\ref{thm:safety} shows that $h_x(x)\ge0$, i.e. $S_x$ is safe (see also Example 8 in \cite{notomista2020persistification}).

\begin{remark}
\label{rmk:reldegree}
If $h_x$ has relative degree greater than 1 with respect to $u$, then recursive or exponential CBFs can be leveraged. See techniques developed in \cite{nguyen2016exponential} and \cite{notomista2020persistification}.
\end{remark}

To conclude this section, we notice that Theorem~\ref{thm:safety} and Remark~\ref{rmk:reldegree} suggest the definition of the following set of controllers:
\begin{equation}
\label{eq:Kx}
K_x(x,u) = \{ v \in \mc \R^m \colon L_f h_x^\prime(x,u,v) + \gamma_x(h_x^\prime(x,u)) \geq 0\}.
\end{equation}
Theorem~\ref{thm:safety} can be then interpreted as: if $v\in K_x(x,u)$, the set $\mc S_x$ is safe.

\subsection{Ensuring Passivity}
\label{subsec:passivity}

As pointed out before, as passivity is a condition on the control input $u$ rather than the state $x$, in this paper, we employ integral CBFs (I-CBFs) to ensure passivity conditions of a dynamical system. The following result---analogous to Theorem~\ref{thm:safety} for I-CBFs---will be leveraged.

\begin{theorem}[\cite{ames2020integral}]
\label{thm:icbfs} 
Consider the control system \eqref{eq:sysdyn} and suppose there is a corresponding dynamically defined controller $\dot u = \phi(x,u,t)$. If the safe set $\mc S \subset \R^n \times \R^m$ is defined by an integral control barrier function $h : \R^n \times \R^m \to \R$, then modifying the dynamically defined controller to be of the form
\begin{eqnarray}
\label{eq:umodifiedqp}
\dot u  =  \phi(x,u,t) + v^*(x,u,t)
\end{eqnarray}
with $v^*$ the solution of the quadratic program (QP)
\begin{equation}
\label{eq:CBFQPv}
\begin{aligned}
v^*(x,u,t)  =  \argmin_{v \in \R^{m}} ~& \| v \|^2 \\
\subjto &  \frac{\partial h}{\partial u}(x,u) v + \frac{\partial h}{\partial x}(x,u)f(x, u) \\
& + \frac{\partial h}{\partial u}(x,u)\phi(x,u,t) + \gamma(h(x,u)) \ge 0
\end{aligned}
\end{equation}
results in safety, i.e. the control system \eqref{eq:sysdynext} with the dynamically defined controller \eqref{eq:umodifiedqp} results in $\mc S$ being forward invariant: if $(x(0),u(0)) \in \mc S$ then $(x(t),u(t)) \in \mc S$ for all $t \geq 0$.
\end{theorem}

We now define an I-CBF which Lemma~\ref{lem:passafe} shows to be suitable to ensure the passivity of the system \eqref{eq:sysdyn}.

Let $V\colon\R^n\to\R$ be a continuously differentiable positive definite function, and define the following I-CBFs:
\begin{equation}
\label{eq:passivityicbf}
h_u(x,u) := g(x)\tr u - L_f V(x,u).
\end{equation}
The corresponding safe set $\mc S_u$ is defined as
\begin{equation}
\label{eq:passivitysafeset}
S_u = \{(x,u)\in\R^n\times\R^m\colon g(x)\tr u - L_f V(x,u) \ge 0 \}.
\end{equation}

\begin{lemma}[Passivity as safety]
\label{lem:passafe}
Safety of $\mc S_u$ in \eqref{eq:passivitysafeset} $\implies$ Passivity of \eqref{eq:sysdyn}.
\end{lemma}

\begin{proof}
Assume $\mc S_u$ is safe, i.e. $(x,u) \in \mc S_u$ for all $t$.
From \eqref{eq:passivitysafeset}, one has:
\begin{equation}
\label{eq:passivitycond}
g(x)\tr u - L_f V(x,u) = y \tr u - \frac{\de V}{\de x} f(x,u) = y\tr u - \dot V \ge 0
\end{equation}
for all $t$. Thus, $y\tr u \ge \dot V$ and, by Definition~\ref{def:passivity}, the system is passive with storage function $V$ and with respect to input $u$ and output $y$.
\qed
\end{proof}

\begin{remark}
The expression of $h_u$ in \eqref{eq:passivityicbf} represents the power dissipated by the system. In \cite{secchi2006position} and \cite{notomista2019passivity}, methods to ensure the passivity of a system in terms of energy are proposed. While those approaches are more flexible, insofar as they enforce conditions similar to $h_u(x,u)\ge0$ in \eqref{eq:passivityicbf}, they are also more sensitive to parameter tuning (see, for instance, discussions on $T_\mathrm{max}$ in \cite{giordano2013passivity}).
\end{remark}

Similarly to what has been done before, the result in Theorem~\ref{thm:icbfs} suggests the definition of the following set of controllers:
\begin{equation}
\label{eq:Kx}
\begin{aligned}
K_u(x,u) = \bigg\{ v \in \mc \R^m \colon &\frac{\partial h_u}{\partial u}(x,u) v + \frac{\partial h_u}{\partial x}(x,u)f(x, u)\\
&+ \frac{\partial h_u}{\partial u}(x,u)\phi(x,u,t) + \gamma(h_u(x,u)) \ge 0 \bigg\}.
\end{aligned}
\end{equation}
By Theorem~\ref{thm:icbfs}, the safety of $\mc S_u$---and, by Lemma~\ref{lem:passafe}, the passivity of \eqref{eq:sysdyn}---is enforced using the I-CBF $h_u$ by picking a controller in $K_u(x,u)$.

With this result in place, we are now ready to combine safety and passivity. Before presenting the safety and passivity filter, in the following section we show how to ensure that the dynamically extended system \eqref{eq:sysdynext} asymptotically behaves as the original system \eqref{eq:sysdyn} when safety constraints are not violated.

\subsection{Tracking of Desired Control Inputs}
\label{subsec:tracking}

The dynamic extension of the system \eqref{eq:sysdyn} proposed in \eqref{eq:sysdynext} is required in order to enforce constraints on the input $u$---the passivity constraints---through a proper choice of $v$. On the other hand, due to this extension, we are not able to control the original system \eqref{eq:sysdyn} using $u$ anymore, but rather we have to design a suitable function $\phi$ in \eqref{eq:sysdynext} in order to \textit{track} the desired $u$ using $v$. The objective of this section is that of presenting a controller that serves this purpose\footnote[3]{It is worth noticing that there are cases in which a dynamically defined controller $\dot u$ is already available (see, for instance, \cite{wardi2019tracking}).}.

Assume we want $u=u\fb(x) + u\h$ as in \eqref{eq:input}. As $\dot u = \phi(x,u) + v$, one could set
\begin{equation}
\label{eq:trackingudot}
\dot u = \dot u\fb(x) + \dot u\h + v^*= \underbrace{L_f u\fb(x,u) + \dot u\h}_{=: \phi(x,u,t)} + v^*,
\end{equation}
where $v^*$ is given by \eqref{eq:CBFQPv} \cite{notomista2020long}. This choice, however, may cause $u(t)$ to diverge over time more and more from its desired value $u\fb(x(t)) + u\h(t)$, due to the fact that $v^*$ from \eqref{eq:CBFQPv} is the minimizer of the difference between the time derivatives of the input functions. In fact, from \eqref{eq:umodifiedqp},
\begin{equation}
\| v^* \| = \|\dot u - \phi(x,u,t)\|,
\end{equation}
is the norm of the difference between the derivative of $u$---rather than the input function itself---and $\phi$. The following theorem presents a dynamically defined control law which results in $u(t)$ converging to the desired value $u\fb(x(t)) + u\h(t)$ as $t\to\infty$ whenever safety is not violated.

\begin{proposition}
\label{prop:tracking}
Consider the system \eqref{eq:sysdyn} and a desired nominal input signal \eqref{eq:input}. Consider an I-CBF $h \colon \R^n \times \R^m \to \R$ defined to ensure the safety of the set $\mc S \subset \R^n \times \R^m$ defined as its zero superlevel set. Then, the dynamically defined controller
\begin{equation}
\label{eq:utrack}
\dot u = \underbrace{L_f u\fb(x,u) + \dot u\h + \frac{\alpha}{2}(u\fb(x)+u\h-u)}_{=:\phi(x,u,t)} + v^*,
\end{equation}
where $\alpha > 0$ and $v^*$ is given in \eqref{eq:CBFQPv}, will ensure the safety of the set $\mc S$, as well as the tracking of the nominal control signal \eqref{eq:input} whenever the controller $\phi(x,u,t)$ is safe.
\end{proposition}

\begin{proof}
First of all, by $\phi(x,u,t)$ being safe we mean that the constraints in \eqref{eq:CBFQPv} are inactive and, consequently, $v^*(x,u,t)=0$. Then, by Theorem~\ref{thm:icbfs}, the controller \eqref{eq:utrack} results in the forward invariance, i.e. safety, of the set $\mc S$. Therefore, we only need to confirm that, if the controller \eqref{eq:utrack} is safe, then $u$ will track the nominal controller \eqref{eq:input}. To this end, let us consider the following Lyapunov function candidate for the system in \eqref{eq:sysdynext} with $\phi(x,u,t)$ given in \eqref{eq:trackingudot}:
\begin{equation}
W(u,x,u\h) = \frac{1}{2} \| u - u\fb(x) - u\h \|^2.
\end{equation}
Its time derivative evaluates to:
\begin{equation}
\begin{aligned}
\dot W &= \frac{\partial W}{\partial u}\dot u + \frac{\partial W}{\partial x}\dot x + \frac{\partial W}{\partial u\h}\dot u\h\\
&= (u - u\fb(x) - u\h)\tr \left( \dot u - L_f u\fb(x,u) - \dot u\h \right)\\
\end{aligned}
\end{equation}
Substituting the proposed controller \eqref{eq:utrack}, we obtain
\begin{equation}
\begin{aligned}
\dot W &= (u - u\fb(x) - u\h)\tr \\
&\quad\left( L_f u\fb(x,u) + \dot u\h + \frac{\alpha}{2}(u\fb(x)+u\h-u) - L_f u\fb(x,u) - \dot u\h \right)\\
&= \frac{\alpha}{2}(u - u\fb(x) - u\h)\tr (u\fb(x)+u\h-u)\\
&=-\alpha W(u,x,u\h).
\end{aligned}
\end{equation}
Thus, $W(t)\to0$, or equivalently $u(t)\to u\fb(x(t)) + u\h(t)$, as $t\to\infty$, i.e. the input $u$ will track the desired control signal \eqref{eq:input}.
\qed
\end{proof}

\begin{remark}
The variable $\dot u$ only appears in the software implementation of the passive and safe controller for the system \eqref{eq:sysdyn}. The actual input given to the system is its integral $u(t)$. Therefore, the value $\alpha$ in the expression of the dynamically defined controller \eqref{eq:utrack} can be chosen arbitrarily large, being aware of not introducing rounding or numerical errors while solving the QP \eqref{eq:CBFQPv}. As can be noticed in the proof of Proposition~\ref{prop:tracking}, the larger the value of $\alpha$ is, the faster the convergence of $u$ to the desired controller \eqref{eq:input} when $v^*=0$ (i.e. when no safety-related modifications of $\dot u$ are required).
\end{remark}

\begin{remark}
\label{rmk:integral}
Integrating the expression of the dynamically defined controller in \eqref{eq:utrack} with respect to time, we get:
\begin{align}
u(t) &= \int_0^t \left( L_f u\fb(x(\tau),u(\tau)) + \dot u\h(\tau) \right) d\tau + \underbrace{\frac{\alpha}{2}\int_0^t \left( u\fb(x(\tau))+u\h(\tau)-u(\tau) \right)d\tau}_\text{Integral control},
\end{align}
where we explicitly recognize the \textit{integral component} of the dynamically defined controller which ensures the desired tracking properties \cite{wardi2019tracking}.
\end{remark}

\subsection{Safety- and Passivity-preserving Controller Design}
\label{subsec:filter}

In this section, we combine the results of the previous three subsection to design a safety and passivity input filter.

\begin{proposition}[Main result]
\label{prop:filter}
Consider a dynamical system \eqref{eq:sysdyn}, a set $S_x$ where we want the state $x$ of the system to remain confined for all times (\textit{safety}), and a continuously differentiable positive definite function $V$ with respect to which we want the system to be passive (\textit{passivity}). If the controller
\begin{equation}
\label{eq:passafeinputQP}
v^*(x,u,t)  =  \argmin_{v \in K_{xu}(x,u)} ~ \| v \|^2,
\end{equation}
where $K_{xu}(x,u) = K_x(x,u) \cap K_u(x,u) \subset \R^{m}$, exists for all times $t$, then the system \eqref{eq:sysdyn} is safe and passive.
\end{proposition}

\begin{proof}
The proof of this proposition is based on the combination of the results of Theorems~\ref{thm:safety} and \ref{thm:icbfs} with Lemma~\ref{lem:passafe}.

If the QP \eqref{eq:passafeinputQP} has a solution for all $t$, then $v^*(x,u,t)\in K_{xu}(x,u)$ for all $t$. Then, by Theorem~\ref{thm:safety}, as $v^*(x,u,t)\in K_{x}(x,u)$, $S_x$ defined in \eqref{eq:safeset} using $h_x$ is forward invariant, i.e. safe. Moreover, as $v^*(x,u,t)\in K_{u}(x,u)$ for all $t$, Theorem~\ref{thm:icbfs} ensures that $S_u$ defined in \eqref{eq:passivitysafeset} using $h_u$ is safe. Thus, by Lemma~\ref{lem:passafe}, the system \eqref{eq:sysdyn} is passive.
\qed
\end{proof}

\begin{figure}
\centering
\includegraphics[width=0.75\textwidth]{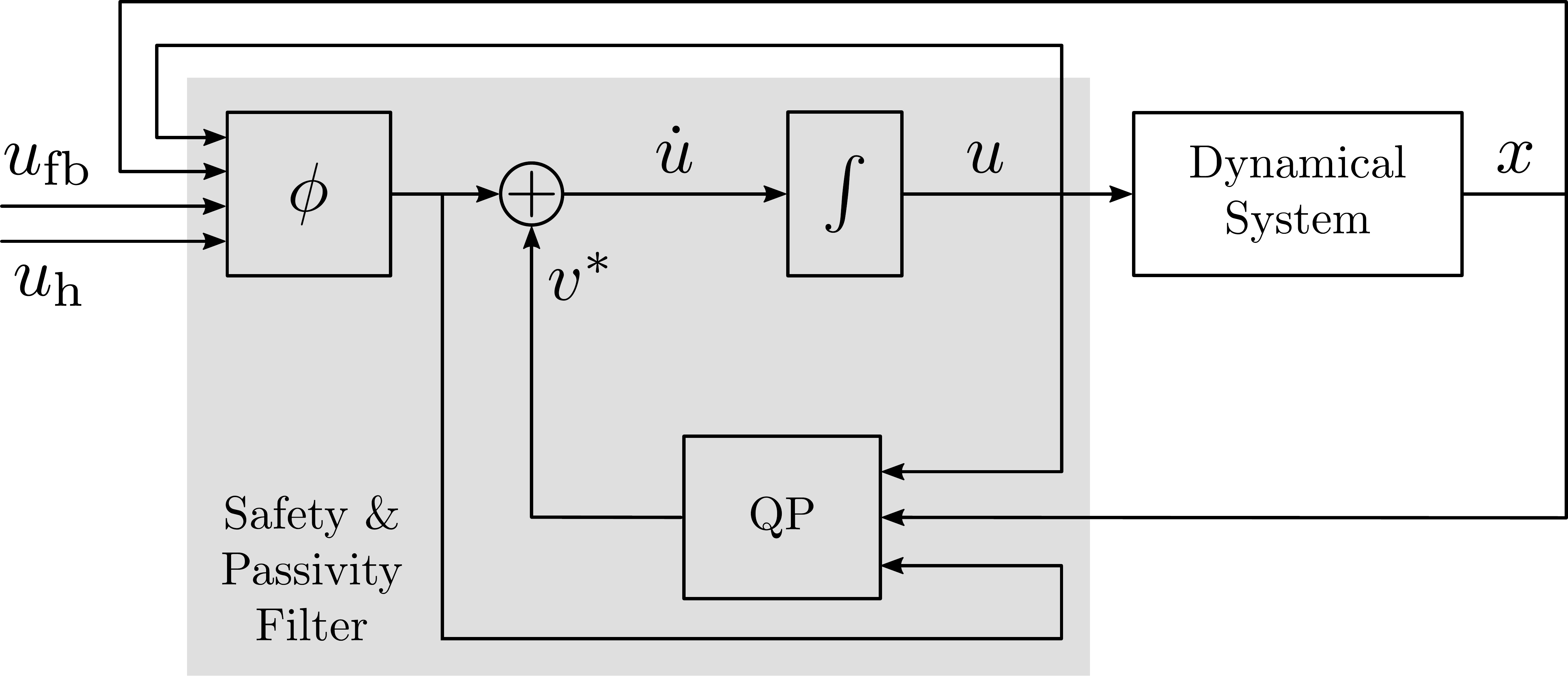}
\caption{Passivity and safety filter: the structure of the filter depicted in Fig.~\ref{fig:passafetyfiltermin} is specified using the results of the paper. The filter includes the computation of the function $\phi$, whose expression is given in \eqref{eq:utrack}, the solution of the convex quadratic program (QP) \eqref{eq:passafeinputQP}, and an integration step, in order to compute the control input $u$ supplied to the system.}
\label{fig:passafetyfilter}
\end{figure}

\begin{remark}[Safety and passivity filter]
Solving the QP \eqref{eq:passafeinputQP} can be interpreted as \textit{filtering} the desired control input given in \eqref{eq:input}---comprised of a state feedback component, $u\fb$, and a human input, $u\h$---to obtain $v^*$. The filtered controller is then integrated in software to obtain the actual control input $u$ supplied to the system \eqref{eq:sysdyn} to ensure its safety and passivity. See Fig.~\ref{fig:passafetyfilter}.

The filtering, i.e. the synthesis of the safety- and passivity-preserving controller, is implemented as an optimization-based controller solution of a convex quadratic program. As such, it can be efficiently solved in online settings, even under real-time constraints, in many robotic applications.
\end{remark}

The following section shows the benefits of the safety and passivity filter developed in this paper applied to the human teleoperation of a second-order dynamical system, modeling a mechanical robotic platform.

\section{Simulation Results}

In this section, we present the results of the application of the safety and passivity filter developed in the previous section to the case of a second-order dynamical system controlled both by a feedback controller and by an external control input of an operator.

The model of the system is the following:
\begin{equation}
\label{eq:sim_sys}
\begin{cases}
\dot x_1 = x_2\\
\dot x_2 = -\sigma x_2 + u\\
y = x_2,
\end{cases}
\end{equation}
where $x_1,x_2,u,y\in\R^2$, $\sigma>0$. Its dynamic extension \eqref{eq:sysdynext} is:
\begin{equation}
\label{eq:sim_sysdynext}
\begin{cases}
\dot x_1 = x_2\\
\dot x_2 = -\sigma x_2 + u\\
\dot u = \phi(x,u,t) + v\\
y = x_2,
\end{cases}
\end{equation}
where $v\in\R^2$. The desired input $\hat u$ is a PD controller aimed at driving the state of the system to the origin:
\begin{equation}
\label{eq:sim_u}
\hat u = u\fb(x) + u\h = -k_P x_1 - k_D x_2 + u\h,
\end{equation}
where $k_P,k_D>0$ are the proportional and derivative control gains, and the human input has been set to $u\h=[-0.3,0]\tr$. From the desired \eqref{eq:sim_u} the expression of $\phi(x,u,t)$ can be obtained using \eqref{eq:utrack}:
\begin{equation}
\phi(x,u,t)  = -k_P x_2 - k_D (-\sigma x_2 + u) + k_I (\hat u - u) + \dot u\h,
\end{equation}
where $k_I>0$ plays the role of $\alpha$ in \eqref{eq:utrack}, i.e. an integral gain, as noticed in Remark~\ref{rmk:integral}.

To ensure the passivity of the system, the following storage function has been employed:
\begin{equation}
V\colon\R^n\to\R\colon x \mapsto \|x\|^2.
\end{equation}
and the I-CBF $h_u$ \eqref{eq:passivityicbf} has been employed. For the system \eqref{eq:sim_sys}, the passivity condition \eqref{eq:passivitycond} becomes:
\begin{equation}
A_u v \le b_u,
\end{equation}
where
\begin{align}
A_u(x) = &x_2\tr\\
b_u(x,u,t) = &-(1+3\sigma^2) \|x_2\|^2 + 2\sigma x_1\tr x_2 - \left(2x_1\tr - 3\sigma x_2\tr\right)u \\
&-x_2\tr \phi(x,u,t)+\gamma_u\left(2\sigma\|x_2\|^2-2x_1\tr x_2-x_2\tr u\right).
\end{align}

Safety has been defined as the condition that $x_1$ never enters the unit disk centered at the origin. To this end, the following CBF has been defined:
\begin{equation}
h_x(x) = \|x_1\|^2 - d^2,
\end{equation}
where $d=1$. As $h_x$ has relative degree 2 with respect to $u$, a recursive approach has been employed, as discussed in Remark~\ref{rmk:reldegree}. Thus, the following two auxiliary CBFs arise:
\begin{align}
h_x^\prime &= 2x_1\tr x_2 + \|x_1\|^2 - d^2\\
h_x^{\prime\prime} &= \|x_1\|^2 + 2\|x_2\|^2 +2(2-\sigma)x_1\tr x_2 + 2x_1\tr u - d^2,
\end{align}
and the safety condition \eqref{eq:cbfdefinition} becomes:
\begin{equation}
A_x v \le b_x,
\end{equation}
where 
\begin{align}
A_x(x) = &-x_1\tr\\
b_x(x,u,t)= & (2x_1\tr+2(2-\sigma)x_2\tr+2u\tr) x_2 + (4x_2\tr+2(2-\sigma)x_1\tr)(-\sigma x_2 + u) \\
&+ 2x_1\tr\phi(x,u,t) + \gamma_x(h_x^{\prime\prime}).
\end{align}

Passivity and safety conditions are then combined in the following single QP equivalent to \eqref{eq:passafeinputQP}:
\begin{align}
\label{eq:sim_passafeinputQP}
v^*(x,u,t)  =  \argmin_{v \in \R^2} ~ &\| v \|^2\\
\subjto & A_x(x) v \le b_x(x,u,t) ~~~\text{(Safety constraint)}\\
& A_u(x) v \le b_u(x,u,t) ~~~\text{(Passivity constraint)}.
\end{align}
The result of the implementation of the solution of \eqref{eq:sim_passafeinputQP} to control the system \eqref{eq:sim_sysdynext} are reported in the following.

\begin{figure}
\centering
\subfloat[][]{\label{subfig:sim_nothing:a}\includegraphics[width=0.35\textwidth]{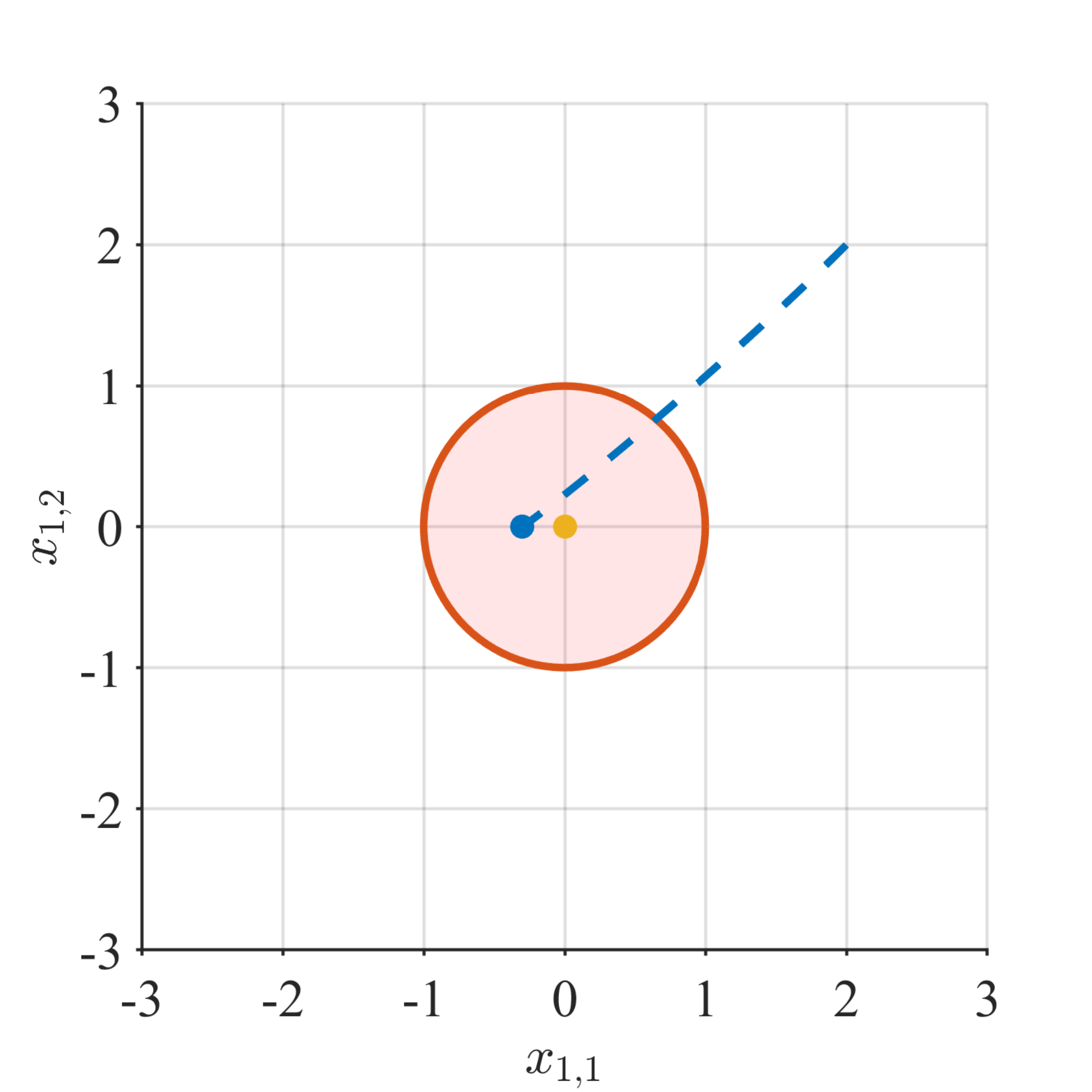}}\hfill
\subfloat[][]{\label{subfig:sim_nothing:b}\includegraphics[width=0.6\textwidth]{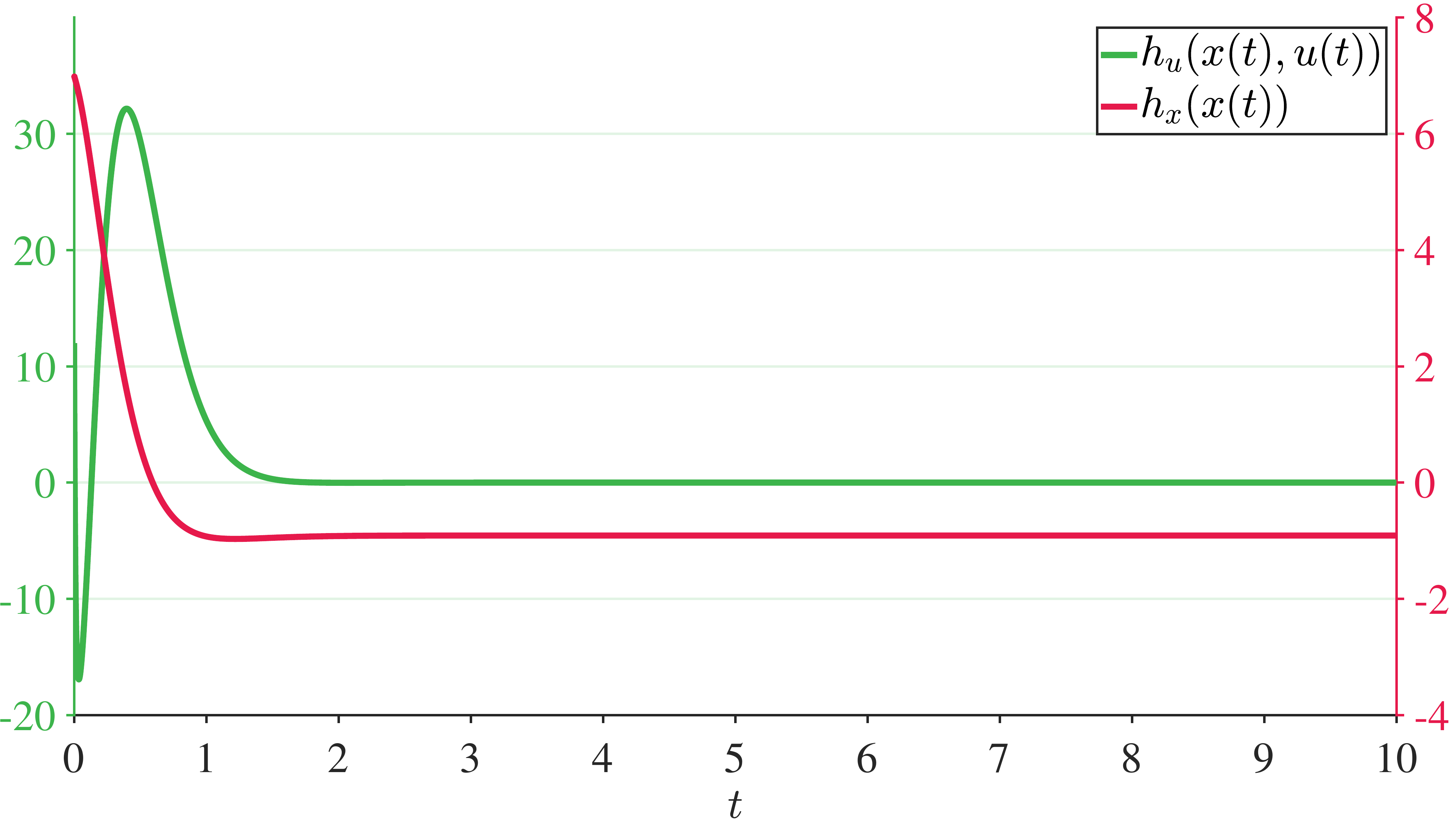}}
\caption{Trajectory (Fig.~\protect\ref{subfig:sim_nothing:a}) and passivity I-CBF (Fig.~\protect\ref{subfig:sim_nothing:b}) for the system \eqref{eq:sim_sysdynext} controlled using $v(t)=0$ for all $t$. The state $x_1$ converges to $[x_{1,1},x_{1,2}]=[-0.3,0]$ as expected, however both the safety and the passivity conditions are violated. In fact, the blue point enters the unsafe region shaded in red in Fig.~\protect\ref{subfig:sim_nothing:a}, and both $h_u$ and $h_x$ take negative values in Fig.~\protect\ref{subfig:sim_nothing:b}.}
\label{fig:sim_nothing}
\end{figure}

Figure~\ref{fig:sim_nothing} shows the trajectory of the state and of the I-CBF $h_u$ for the system \eqref{eq:sim_sysdynext} controlled using $v(t)=0$ for all $t$. As no safety constraint is enforced, the system trajectory enter the red disk (unsafe region). Moreover, as no passivity constraint is enforced, the value of $h_u$ becomes negative. From \eqref{eq:passivityicbf}, this implies that $y\tr u \not\ge \dot V$, i.e. energy is generated and the system is not passive.

\begin{figure}
\centering
\subfloat[][]{\label{subfig:sim_passive:a}\includegraphics[width=0.35\textwidth]{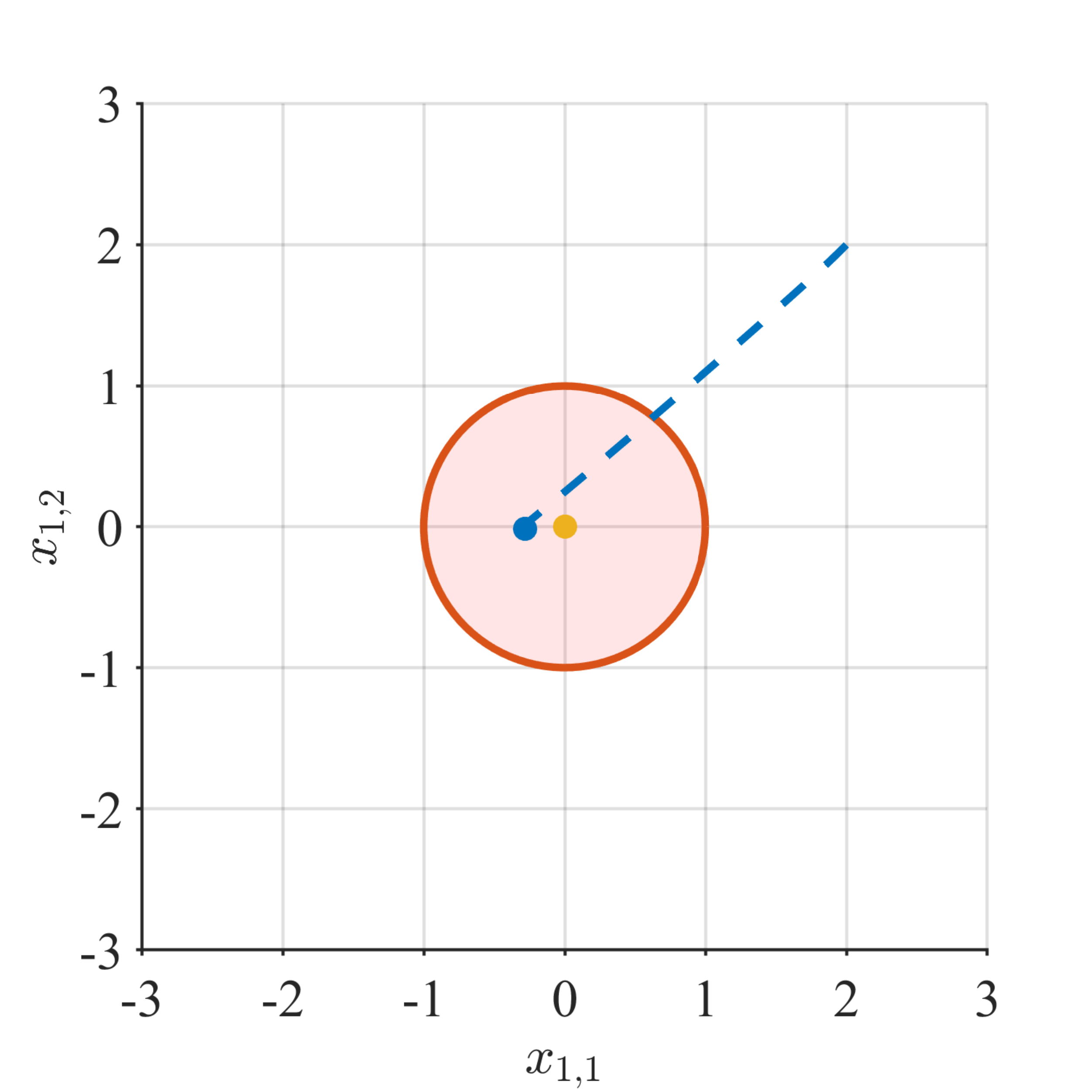}}\hfill
\subfloat[][]{\label{subfig:sim_passive:b}\includegraphics[width=0.55\textwidth]{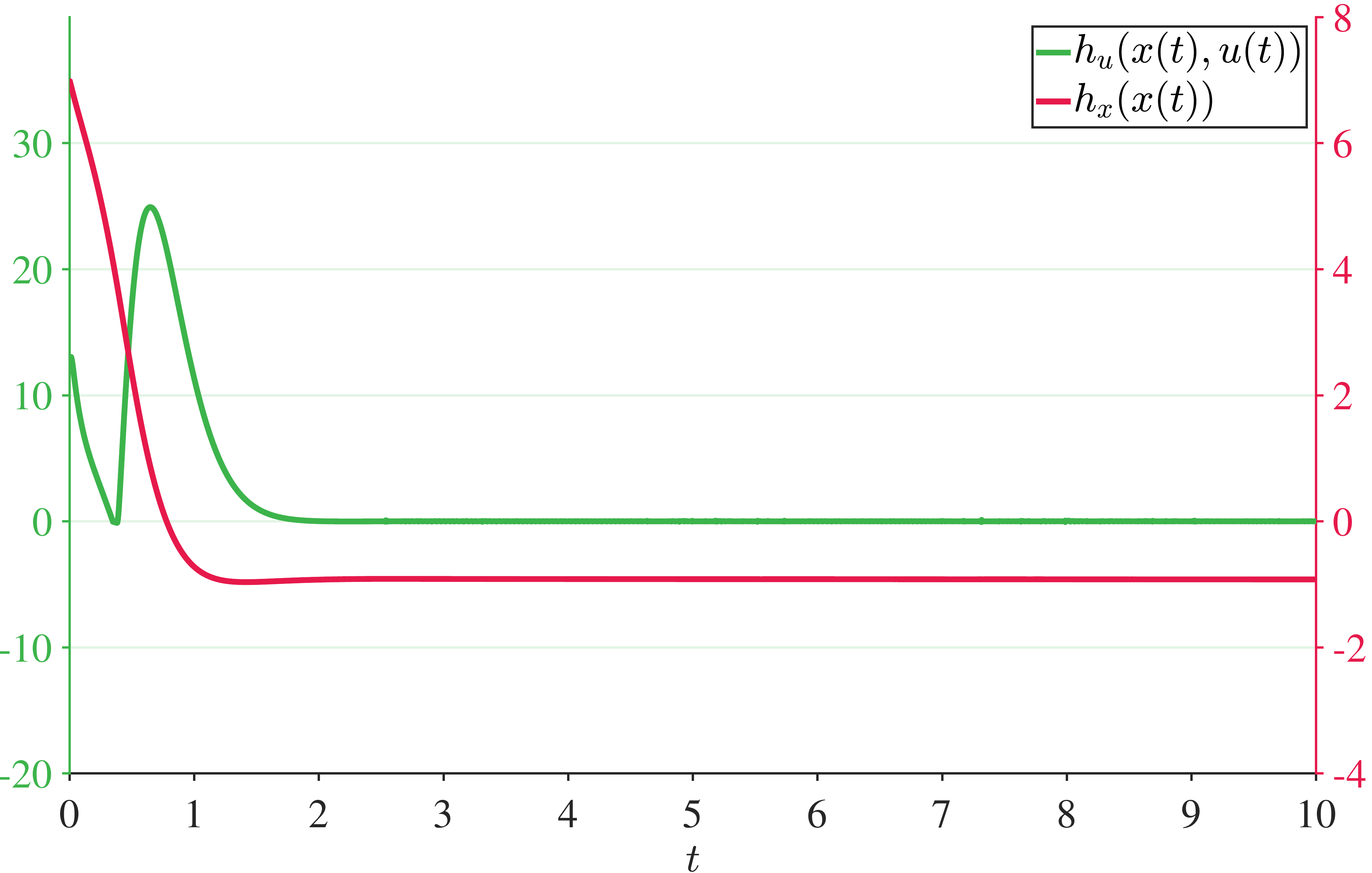}}
\caption{Trajectory (Fig.~\protect\ref{subfig:sim_passive:a}) and passivity I-CBF (Fig.~\protect\ref{subfig:sim_passive:b}) for the system \eqref{eq:sim_sysdynext} controlled using $v(t)=v^*$ solution of the QP \eqref{eq:sim_passafeinputQP} without the safety constraint $A_x(x) v \le b_x(x,u,t)$. The state $x_1$ converges to $[x_{1,1},x_{1,2}]=[-0.3,0]$ as expected and, in addition to the simulation in Fig.~\ref{fig:sim_nothing}, it does so while preserving passivity for all times. In fact, in Fig.~\protect\ref{subfig:sim_nothing:b}, it can be seen how $h_u$ is always kept positive. Values of $h_x$, on the other hand, become negative when the blue dot in Fig.~\protect\ref{subfig:sim_passive:a} is inside the red-shaded disk.}
\label{fig:sim_passive}
\end{figure}

To mitigate this issue, we introduce the passivity I-CBF constraint $A_u(x) v \le b_u(x,u,t)$. Figure~\ref{fig:sim_passive} shows the results of the implementation of the controller $v^*$ solution of the QP \eqref{eq:sim_passafeinputQP} without the safety constraint $A_x(x) v \le b_x(x,u,t)$. The trajectory of the system still enters the unsafe red-shaded region, however the value of $h_u$ is always positive. For the same rationale discussed above, in this case energy is not generated and $y\tr u \ge \dot V$ as desired, i.e. the system is passive.

\begin{figure}
\centering
\subfloat[][]{\label{subfig:sim_passafeinputQP:a}\includegraphics[width=0.35\textwidth]{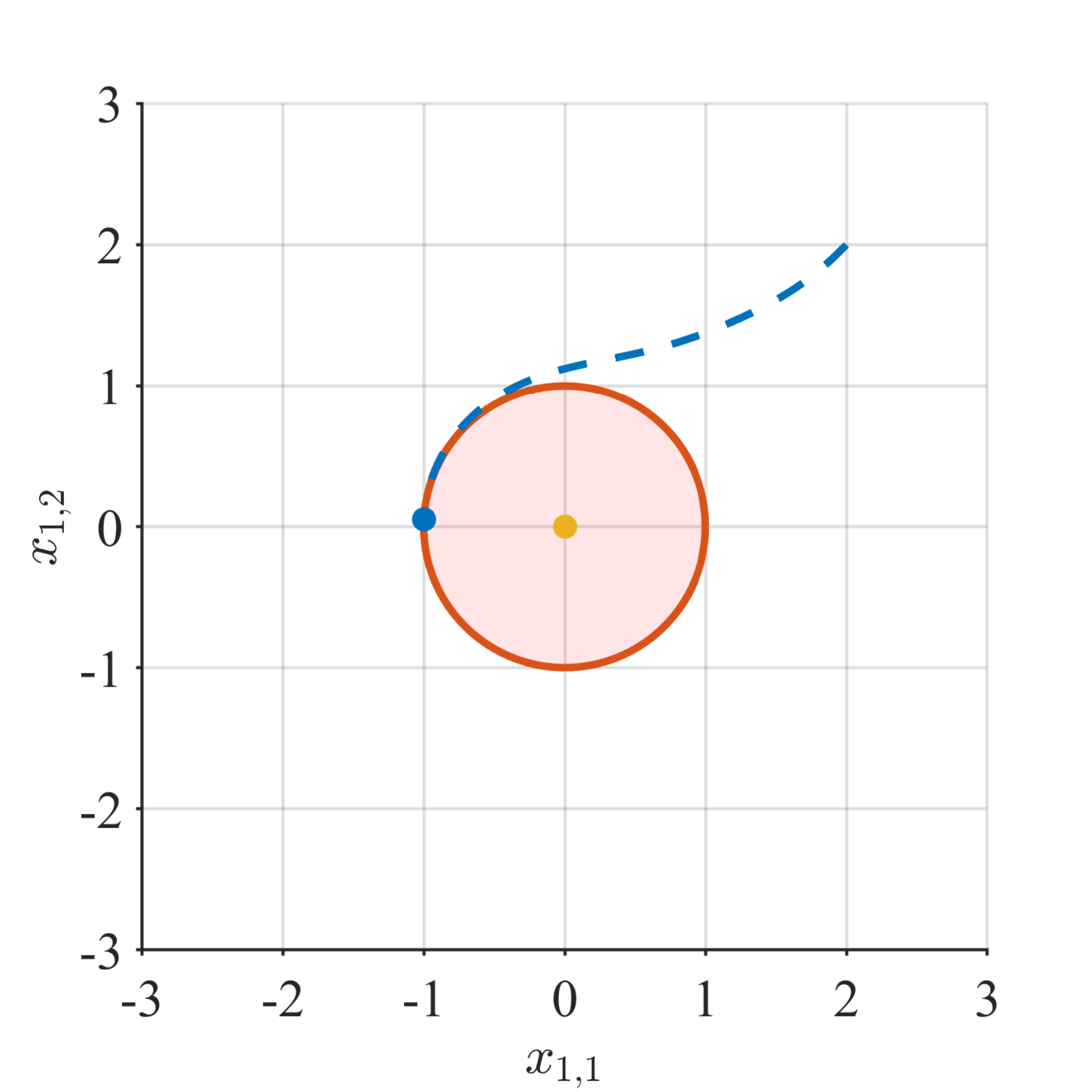}}\hfill
\subfloat[][]{\label{subfig:sim_passafeinputQP:b}\includegraphics[width=0.55\textwidth]{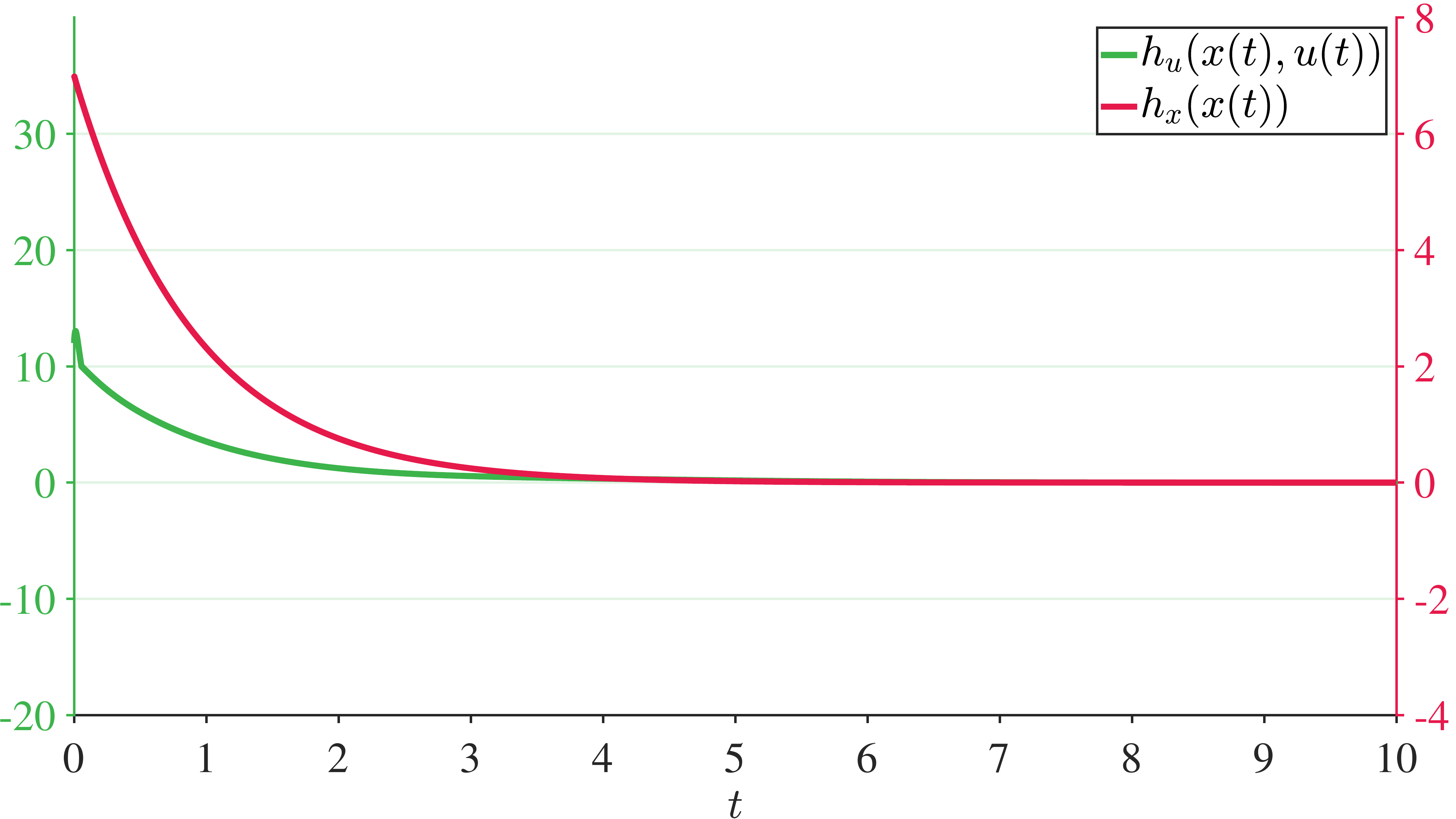}}
\caption{Trajectory (Fig.~\protect\ref{subfig:sim_passafeinputQP:a}) and passivity I-CBF (Fig.~\protect\ref{subfig:sim_passafeinputQP:b}) for the system \eqref{eq:sim_sysdynext} controlled using $v(t)=v^*$ solution of the QP \eqref{eq:sim_passafeinputQP} with safety and passivity constraints. By enforcing the safety constraints, the blue dot in Fig.~\protect\ref{subfig:sim_passafeinputQP:a} is not allowed to enter the unsafe region (red disk). As a result, $x_1$ does not converge to $[x_{1,1},x_{1,2}]=[-0.3,0]$ as desired by the human input, but in this case it reaches the value of $x_1$ in the safe region closest to $[-0.3,0]$, i.e. $[-1,0]$. Moreover, this resulting safe trajectory is obtained by executing only safe actions: Fig.~\protect\ref{subfig:sim_nothing:b} shows how safety and passivity are preserved in terms of the values of $h_x$ and $h_u$ being both kept positive.}
\label{fig:sim_passafe}
\end{figure}

Finally, to show how safety and passivity constraints can be enforced in a holistic fashion, Fig.~\ref{fig:sim_passafe} shows the behavior of the system controlled by the solution of the complete QP \eqref{eq:sim_passafeinputQP}. The trajectory of the system is kept away from the unsafe region by the effect of the safety constraints and, at the same time, the value of $h_u$ remains positive for all times, i.e. the system is safe.

\section{Conclusions and Future Work}

In this paper, we introduced a safety and passivity filter which is able to guarantee that a dynamical system remains passive and a subset of its state space remains forward invariant. This technique is particularly suitable in robot teleoperation scenarios where a human is interconnected---by an input-output relation---to a robotic system and exchange energy with the system through the supplied control inputs. The passivity of the interconnected system guarantees that no energy is generated by the interconnection, while the forward invariance property ensures the safety of the interaction between the human operator and the robotic system. 

Future work will be devoted to the feasibility analysis of the optimization problem which defines the safety and passivity filter, as well as to the introduction of estimation algorithms required to evaluate the input supplied by the human interacting with the robotic system. Moreover, in this paper, we shown the approach applied to a simulated linear second-order system, representing a simple robotic system, controlled by a feedback controller as well as a human input. Future work will focus on applying this method to real manipulator robots and multi-robot systems, which are commonly employed in robot teleoperation applications.

\bibliographystyle{bib/spmpsci}
\bibliography{bib/hfr2020ref}

\end{document}

%% file: packages_commands/mypackages.tex
\usepackage{csquotes}
\usepackage{comment}
\usepackage{enumerate}
\usepackage{pbox}

\usepackage{graphicx}
\graphicspath{{figures/}}
\usepackage{subfig}
\usepackage{epsfig}
\usepackage[inkscapepath=figures/svgtopdf]{svg}
\usepackage{import}
\usepackage{xcolor} 
\usepackage{pgfplots}
\pgfplotsset{compat=1.3}
\usetikzlibrary{decorations.markings}
\usetikzlibrary{external}
\usetikzlibrary{arrows}
\usetikzlibrary{automata}
\usepgfplotslibrary{fillbetween}
\usepackage{amsmath}
\usepackage{amssymb}

\usepackage{amsthm}
\usepackage{mathrsfs}
\usepackage{mathtools}
\usepackage{yhmath}
\usepackage{array}
\usepackage{xifthen}
\usepackage{algorithm}
\usepackage{algpseudocode}
\usepackage{bbm}

%% file: packages_commands/mycommands.tex
\DeclareMathOperator*{\argmin}{arg\!\min}
\DeclareMathOperator*{\subjto}{subject\,to~}
\newcommand{\R}{\mathbb R}
\newcommand{\de}{\partial}
\newcommand{\mc}{\mathcal}
\newcommand{\tr}{^T}
\newcommand{\fb}{_\mathrm{fb}}
\newcommand{\h}{_\mathrm{h}}

%% file: hfr2020.bbl
\begin{thebibliography}{10}
\providecommand{\url}[1]{{#1}}
\providecommand{\urlprefix}{URL }
\expandafter\ifx\csname urlstyle\endcsname\relax
  \providecommand{\doi}[1]{DOI~\discretionary{}{}{}#1}\else
  \providecommand{\doi}{DOI~\discretionary{}{}{}\begingroup
  \urlstyle{rm}\Url}\fi

\bibitem{ames2020integral}
Ames, A., Notomista, G., Wardi, Y., Egerstedt, M.: Integral control barrier
  functions for dynamically defined control laws.
\newblock Control Systems Letters  (2020)

\bibitem{ames2019control}
{Ames}, A.D., {Coogan}, S., {Egerstedt}, M., {Notomista}, G., {Sreenath}, K.,
  {Tabuada}, P.: Control barrier functions: Theory and applications.
\newblock In: European Control Conference, pp. 3420--3431 (2019).
\newblock \doi{10.23919/ECC.2019.8796030}

\bibitem{ames2014control}
Ames, A.D., Grizzle, J.W., Tabuada, P.: Control barrier function based
  quadratic programs with application to adaptive cruise control.
\newblock In: Conference on Decision and Control, pp. 6271--6278. IEEE (2014)

\bibitem{anderson1989bilateral}
Anderson, R.J., Spong, M.W.: Bilateral control of teleoperators with time
  delay.
\newblock IEEE Transactions on Automatic control \textbf{34}(5), 494--501
  (1989)

\bibitem{chopra2006passivity}
Chopra, N., Spong, M.W.: Passivity-based control of multi-agent systems.
\newblock In: Advances in robot control, pp. 107--134. Springer (2006)

\bibitem{duindam2004port}
Duindam, V., Stramigioli, S.: Port-based asymptotic curve tracking for
  mechanical systems.
\newblock European Journal of Control \textbf{10}(5), 411--420 (2004)

\bibitem{giordano2013passivity}
Giordano, P.R., Franchi, A., Secchi, C., B{\"u}lthoff, H.H.: A passivity-based
  decentralized strategy for generalized connectivity maintenance.
\newblock The International Journal of Robotics Research \textbf{32}(3),
  299--323 (2013)

\bibitem{hatanaka2015passivity}
Hatanaka, T., Chopra, N., Spong, M.W.: Passivity-based control of robots:
  Historical perspective and contemporary issues.
\newblock In: 2015 54th IEEE Conference on Decision and Control (CDC), pp.
  2450--2452. IEEE (2015)

\bibitem{huang2019guaranteed}
Huang, Y., Yong, S.Z., Chen, Y.: Guaranteed vehicle safety control using
  control-dependent barrier functions.
\newblock In: American Control Conference, pp. 983--988. IEEE (2019)

\bibitem{khalil2015nonlinear}
Khalil, H.K.: Nonlinear control.
\newblock Pearson New York (2015)

\bibitem{nguyen2016exponential}
Nguyen, Q., Sreenath, K.: Exponential control barrier functions for enforcing
  high relative-degree safety-critical constraints.
\newblock In: American Control Conference, pp. 322--328. IEEE (2016)

\bibitem{niemeyer1991stable}
Niemeyer, G., Slotine, J.J.: Stable adaptive teleoperation.
\newblock IEEE Journal of oceanic engineering \textbf{16}(1), 152--162 (1991)

\bibitem{niemeyer2004telemanipulation}
Niemeyer, G., Slotine, J.J.E.: Telemanipulation with time delays.
\newblock The International Journal of Robotics Research \textbf{23}(9),
  873--890 (2004)

\bibitem{notomista2020long}
Notomista, G.: Long-duration robot autonomy: From control algorithms to robot
  design.
\newblock Ph.D. thesis, Georgia Institute of Technology (2020).
\newblock
  \urlprefix\url{https://www.gnotomista.com/files/notomista_gennaro_phd_thesis.pdf}

\bibitem{notomista2019passivity}
Notomista, G., Cai, X., Yamauchi, J., Egerstedt, M.: Passivity-based
  decentralized control of multi-robot systems with delays using control
  barrier functions.
\newblock In: International Symposium on Multi-Robot and Multi-Agent Systems,
  pp. 231--237. IEEE (2019)

\bibitem{notomista2020persistification}
Notomista, G., Egerstedt, M.: Persistification of robotic tasks.
\newblock Transactions on Control Systems Technology  (2020)

\bibitem{ohnishi2019barrier}
Ohnishi, M., Wang, L., Notomista, G., Egerstedt, M.: Barrier-certified adaptive
  reinforcement learning with applications to brushbot navigation.
\newblock Transactions on Robotics \textbf{35}(5), 1186--1205 (2019)

\bibitem{olfati2004consensus}
Olfati-Saber, R., Murray, R.M.: Consensus problems in networks of agents with
  switching topology and time-delays.
\newblock IEEE Transactions on automatic control \textbf{49}(9), 1520--1533
  (2004)

\bibitem{secchi2012bilateral}
Secchi, C., Franchi, A., B{\"u}lthoff, H.H., Giordano, P.R.: Bilateral
  teleoperation of a group of uavs with communication delays and switching
  topology.
\newblock In: 2012 IEEE International Conference on Robotics and Automation,
  pp. 4307--4314. IEEE (2012)

\bibitem{secchi2006position}
Secchi, C., Stramigioli, S., Fantuzzi, C.: Position drift compensation in
  port-hamiltonian based telemanipulation.
\newblock In: 2006 IEEE/RSJ International Conference on Intelligent Robots and
  Systems, pp. 4211--4216. IEEE (2006)

\bibitem{secchi2007control}
Secchi, C., Stramigioli, S., Fantuzzi, C.: Control of interactive robotic
  interfaces: A port-Hamiltonian approach, vol.~29.
\newblock Springer Science \& Business Media (2007)

\bibitem{sieber2018human}
Sieber, D., Hirche, S.: Human-guided multirobot cooperative manipulation.
\newblock IEEE Transactions on Control Systems Technology \textbf{27}(4),
  1492--1509 (2018)

\bibitem{stramigioli2002geometric}
Stramigioli, S., Van Der~Schaft, A., Maschke, B., Melchiorri, C.: Geometric
  scattering in robotic telemanipulation.
\newblock IEEE Transactions on Robotics and Automation \textbf{18}(4), 588--596
  (2002)

\bibitem{wang2017safety}
Wang, L., Ames, A.D., Egerstedt, M.: Safety barrier certificates for
  collisions-free multirobot systems.
\newblock Transactions on Robotics \textbf{33}(3), 661--674 (2017)

\bibitem{wardi2019tracking}
Wardi, Y., Seatzu, C., Cortes, J., Egerstedt, M., Shivam, S., Buckley, I.:
  Tracking control by the newton-raphson method with output prediction and
  controller speedup.
\newblock arXiv preprint arXiv:1910.00693  (2019)

\bibitem{wohlers2017lumped}
Wohlers, M.R.: Lumped and distributed passive networks: a generalized and
  advanced viewpoint.
\newblock Academic press (2017)

\bibitem{wu2016safety}
Wu, G., Sreenath, K.: Safety-critical control of a planar quadrotor.
\newblock In: American Control Conference, pp. 2252--2258. IEEE (2016)

\bibitem{yamauchi2017passivity}
Yamauchi, J., Atman, M.W.S., Hatanaka, T., Chopra, N., Fujita, M.:
  Passivity-based control of human-robotic networks with inter-robot
  communication delays and experimental verification.
\newblock In: 2017 IEEE International Conference on Advanced Intelligent
  Mechatronics (AIM), pp. 628--633. IEEE (2017)

\bibitem{zampieri2008trends}
Zampieri, S.: Trends in networked control systems.
\newblock IFAC Proceedings Volumes \textbf{41}(2), 2886--2894 (2008)

\end{thebibliography}
